\documentclass[12pt]{amsart}

\usepackage[margin={1.75in, 1in}]{geometry}

\title[Information complexity of proper learners for VC classes]{On the information complexity of proper learners for VC classes in the realizable case}

\date{}

\author[M. Haghifam]{Mahdi Haghifam$^{1,2,3}$}
\author[G. K. Dziugaite]{Gintare Karolina Dziugaite$^{1,4}$}
\author[S. Moran]{Shay Moran$^{5}$}
\author[D. M. Roy]{Daniel M. Roy$^{2,3}$}

\thanks{$^1$ Element AI\ \ $^2$ University of Toronto\ \ $^3$ Vector Institute\ \ $^4$ Mila\ \ $^5$ Technion}

\usepackage{microtype}

\usepackage[T1]{fontenc}
\usepackage[utf8]{inputenc}
\usepackage{mathscinet}
\pdfoutput=1
\usepackage{amsmath,amsthm,wrapfig}
	\usepackage[utf8]{inputenc} 
\usepackage{amsmath, amssymb,bm, cases, mathtools, thmtools}
\usepackage{verbatim}
\usepackage{graphicx}\graphicspath{{figures/}}
\usepackage{multicol}
\usepackage{tabularx}
\usepackage[usenames,dvipsnames]{xcolor}
\usepackage{mathrsfs} 
\usepackage{caption}
\usepackage{algorithm}
\usepackage{algorithmicx}
\usepackage[noend]{algpseudocode}
\usepackage{array,booktabs,arydshln,xcolor}

\usepackage[%
		minnames=1,maxnames=99,maxcitenames=2,
		style=authoryear,%
		sortcites=true,
		doi=false,url=false,
		giveninits=true,
		hyperref,natbib,dashed=false,backend=biber]{biblatex}
\renewbibmacro{in:}{%
	\ifentrytype{article}{}{\printtext{\bibstring{in}\intitlepunct}}}

\usepackage[T1]{fontenc}
\usepackage{inconsolata}
\usepackage{xmathptmx}
\usepackage{dsfont}

\usepackage{listings} %

\usepackage{enumitem}
\usepackage{booktabs}       %
\usepackage{nicefrac}       %

\usepackage{lineno}
\usepackage{bbm}

\usepackage{caption}
\usepackage{subcaption}

 \usepackage[colorlinks,citecolor=blue,urlcolor=blue,linkcolor=RawSienna]{hyperref}
\usepackage{datetime}

\DeclareMathAlphabet\EuRoman{U}{eur}{m}{n}
\SetMathAlphabet\EuRoman{bold}{U}{eur}{b}{n}

\makeatletter
\let\reftagform@=\tagform@
\def\tagform@#1{\maketag@@@{\ignorespaces\textcolor{gray}{(#1)}\unskip\@@italiccorr}}
\renewcommand{\eqref}[1]{\textup{\reftagform@{\ref{#1}}}}
\makeatother

\declaretheorem[style=plain,numberwithin=section,name=Theorem]{theorem}

\declaretheorem[style=remark,qed=$\triangleleft$,sibling=theorem,name=Remark]{remark}
\numberwithin{theorem}{section}

\usepackage{xparse}
\usepackage{xstring}
\usepackage{xspace}

\def\[#1\]{\begin{align}#1\end{align}}
\def\*[#1\]{\begin{align*}#1\end{align*}}

\newcommand{\Dist}{\mathcal D}
\newcommand{\dataspace}{\mathcal Z}
\newcommand\optparen[1]{\ifthenelse{\equal{#1}{}}{}{(#1)}}
\newcommand{\RiskChar}{R}
\newcommand{\Risk}[2]{\RiskChar_{#1}\optparen{#2}}
\newcommand{\EmpRisk}[2]{\hat \RiskChar_{#1}\optparen{#2}}

\newcommand{\Nats}{\mathbb{N}}

\newcommand{\downto}{\!\downarrow\!}

\newcommand{\inspace}{\mathcal X}
\newcommand{\outspace}{\mathcal Y}

\DeclareMathOperator*{\newlim}{\mathrm{lim}\vphantom{\mathrm{infsup}}}

\DeclareMathOperator*{\newmax}{\mathrm{max}\vphantom{\mathrm{infsup}}}

\renewcommand{\lim}{\newlim}

\renewcommand{\max}{\newmax}

\renewcommand{\Pr}{\mathbb{P}}
\def\EE{\mathbb{E}}

\newcommand{\defn}[1]{\textit{#1}}

\newcommand{\XX}{\Reals^{\idim}}
\newcommand{\YY}{K}

\newcommand{\hypothesisclass}{\mathcal{H}}

\newcommand{\e}{\mathrm{e}}

\newcommand{\Alg}{\mathcal{A}}

\newcommand{\lcrx}[4][{-1}]{
	\IfEq{#1}{-1}{\left #2 {{{{#3}}}} \right #4}{
   	\IfEq{#1}{0}{#2 {{{{#3}}}} #4}{
	\IfEq{#1}{1}{\bigl #2 {{{{#3}}}} \bigr #4}{
	\IfEq{#1}{2}{\Bigl #2 {{{{#3}}}} \Bigr #4}{
	\IfEq{#1}{3}{\biggl #2 {{{{#3}}}} \biggr #4}{
	\IfEq{#1}{4}{\Biggl #2 {{{{#3}}}} \Biggr #4}{
    \GenericWarning{"4th argument to lcrx must be -1, 0, 1, 2, 3, or 4"}
    }}}}}}}
\newcommand{\range}[1]{ [#1] }

\newcommand{\samplecomplex}[2]{ \mathcal{M}^{\HS}_\text{prop}(#1,#2) }

\renewcommand{\defn}[1]{\emph{#1}}

\newcommand{\SZB}{\ensuremath{\mathrm{CMI}_{\Dist}(\Alg)}\xspace}

\usepackage[hide=false,setmargin=true,marginparwidth=1.2in]{marginalia}

 \usepackage[colorlinks,citecolor=blue,urlcolor=blue,linkcolor=RawSienna]{hyperref}

\usepackage[capitalize]{cleveref}

\crefname{lemma}{Lemma}{Lemmas}
\crefname{corollary}{Corollary}{Corollaries}
\crefname{theorem}{Theorem}{Theorems}

\newtheorem{problem}{Conjecture}
\crefname{problem}{Conjecture}{Conjectures}

\newtheorem{statement}{Statement}
\crefname{statement}{Statement}{Statement}

\newcommand{\HS}{\mathcal H}

\begin{document}

\maketitle

\begin{abstract}
We provide a negative resolution to a conjecture of \citet{steinke2020openproblem}, by showing that their bound on the conditional mutual information (CMI) of proper learners of Vapnik–Chervonenkis (VC) classes cannot be improved from $d \log n+2$ to $O(d)$, where $n$ is the number of i.i.d.\ training examples. In fact, we exhibit VC classes for which the CMI of any proper learner cannot be bounded by any real-valued function of the VC dimension only.
\end{abstract}

\section{Introduction}

\citet{steinke2020reasoning} 
show that, for every VC class of dimension $d$, 
there exists an empirical risk minimization algorithm whose so-called ``conditional mutual information'' (CMI) is no larger than $d \log n + 2$, given
$n$ i.i.d.\ training samples.
The combination of this CMI bound and their agnostic CMI-based generalization bound 
leads to a bound that is, however, suboptimal,
by a $\log n$ factor.
The suboptimality of their agnostic bound prompts \citet{steinke2020openproblem} to conjecture that the CMI bound for proper learners can be improved to $O(d)$ in both the agnostic and realizable case.

In this short note, we provide a counterexample to this conjecture for proper learners in the realizable case.
The basic obstruction is the existence of VC classes such that, in the realizable case, 
the sample complexity of properly learning an $\epsilon$-approximation with probability at least $1-\delta$ is not in $o(\frac 1 \epsilon \log \frac 1 \epsilon + \frac 1 \epsilon \log \frac 1 \delta)$. 
The existence of a learning algorithm with a CMI bound of $O(d)$ for VC classes contradicts this lower bound. 
We discuss implications in the final section.

\section{Preliminaries}
\renewcommand{\XX}{\mathcal X}
\renewcommand{\YY}{\mathcal Y}
\renewcommand{\SS}{S_n}
\renewcommand{\Alg}{\mathcal A_n}
\newcommand{\TheAlg}{\mathcal A}

Let $Z=(Z_{i,j})_{i \in \{0,1\},\,  j \ge 1}$ be an i.i.d.\ array of random elements in a product space $\XX \times \YY$ with common distribution $\Dist$,
let $U=(U_1,U_2,\dots)$ be a sequence of i.i.d.\ Bernoulli random variables in $\{0,1\}$, independent from $Z$, with $\Pr(U_i=0)=\Pr(U_i=1)=1/2$,
and, for every $n \in \Nats$, let $\SS=(Z_{U_j,j})_{j=1}^{n}$. 

Writing $\SS=((X_1,Y_1),\dots,(X_n,Y_n))$,
the empirical risk of a classifier $h: \XX \to \YY$ is $\EmpRisk{\SS}{h} = n^{-1} | \{i \in [n] : h(X_i) \neq Y_i \}|$,
while its risk is $\Risk{\Dist}{h} = \EE\EmpRisk{\SS}{h}$.
A distribution $\Dist$ is \defn{realizable by a class $\HS \subseteq \XX \to \YY$} if 
there exists $h \in \HS$ such that $\Risk{\Dist}{h} = 0$. Note that, in this case, with probability one, there exists $h \in \HS$ such that $\EmpRisk{\SS}{h} = 0$.   A sequence $((x_1,y_1),\dots,(x_n,y_n))$ is said to be \defn{realizable by $\HS$}, if for some $h \in \HS$, $h(x_i)=y_i$ for all $i \in \range{n}$.

Let $\TheAlg = (\Alg)_{n\ge 1}$ be a learning algorithm, i.e., for each positive integer $n$, a (potentially randomized) map taking $\SS$ to an element of $\XX \to \YY$. We say that $\TheAlg$ is a \emph{proper learner for a class $\HS \subseteq \XX \to \YY$} if the codomain of $\Alg$ is a subset of $\HS$ for every $n$.
\citet{steinke2020reasoning} define the \defn{conditional mutual information of $\Alg$},
denoted $\SZB$,
to be the conditional mutual information $I(\Alg(\SS);U|Z)$. 
Note that this quantity is equivalent to $I(\Alg(\SS);\SS|Z)$ when $\Dist$ is atomless, 
because then $(U_1,\dots,U_n)$ is a.s.\ measurable with respect to $\SS$ and $Z$.
which we can assume to be the case without any loss of generality by taking $\Dist'$ to be the product of $\Dist$ with an atomless distribution and taking $\Alg'$ to strip this new coordinate and return the same hypothesis as $\Alg$.

The main results of \citet{steinke2020reasoning} are generalization bounds in terms of CMI.
For example, when loss is bounded in $[0,1]$, they show %
\begin{equation*}
 \EE [\Risk{\Dist}{\Alg(\SS)} - \EmpRisk{S}{\Alg(\SS)}]  \leq  \sqrt{\frac{2\SZB}{n}},
\end{equation*}
and
\[
\label{eq:ege-interpolate}
 \EE \Risk{\Dist}{\Alg(\SS)} \leq 2 \EE \EmpRisk{S}{\Alg(\SS)} +\frac{3\SZB}{n },
\]
where the expectations are over $\SS \sim \Dist^n$ and the (independent) internal randomness in $\Alg$.

For a class $\HS \subseteq \XX \to \YY$, let $\samplecomplex{\epsilon}{\delta}$ denote the \defn{proper optimal sample complexity} of $(\epsilon,\delta)$-PAC learning $\HS$, 
i.e., $\samplecomplex{\epsilon}{\delta}$ is the least integer $n$ for which there exists a proper learning algorithm $\TheAlg$ such that, for every realizable distribution $\Dist$, 
\begin{equation*}
\Pr(\Risk{\Dist}{\Alg(\SS)} \geq \epsilon) \leq \delta.
\end{equation*}

\section{Conjectures}

\citet{steinke2020openproblem} propose several conjectures regarding CMI for proper learning of VC classes under realizability assumptions. We focus on two of their conjectures, which can be seen as special cases of the following statement:

\begin{statement}
\label{conjecture:realizable-generalf}
There exists a real-valued function $f$ and constant $c \ge 0$ such that, 
for every nonnegative integer $d$ and VC class $\hypothesisclass \subseteq \XX \to \YY$ of dimension $d$,
there exists a proper learning algorithm $\TheAlg$ for $\HS$ 
such that,
for every $n \ge d$, 
$\SZB \leq f(d)$ for all $\Dist$
and, for every realizable $s \in \dataspace^{n}$, 
\begin{equation*}
\EE \EmpRisk{s}{\Alg(s)}  \leq c\, \frac{d}{n},
\end{equation*}
where the expectation is over only the randomness in $\Alg$.
\end{statement}

The following two conjectures were proposed by \citet{steinke2020openproblem}:

\setcounter{problem}{6}
\begin{problem}
\label{conjecture:realizable}
\cref{conjecture:realizable-generalf} holds for $f$ linear.
\end{problem}

\begin{problem}
\label{conjecture:realizable-constant}
\cref{conjecture:realizable-generalf} holds for $f$ linear and $c=0$.
\end{problem}

\citet{steinke2020reasoning} identify a proper learning algorithm for the set of threshold functions in one dimension whose CMI is independent of the size of the training sample for realizable distributions, providing some evidence towards \cref{conjecture:realizable-constant}.

\section{A Limitation of Proper Learning}

In this section, we prove that \cref{conjecture:realizable-generalf} is false, which then implies that \cref{conjecture:realizable,conjecture:realizable-constant} are false.
We begin with some definitions.
Two sequences $((x_1,y_1),\dots,(x_n,y_n))$ and $((x'_1,y'_1),\dots,(x'_n,y'_n))$ are neighbors if $x_i=x'_i$ for all $i \in \range{n}$ and $y_i = y'_i$ for all but exactly one $i \in \range{n}$.
Finally, 
\citet[Def. 3]{bousquet2020proper} define 
the \defn{hollow star number of $\hypothesisclass$}, denoted by $k_o$, 
to be the largest integer $n$ such that there exists $s \in (\inspace \times \outspace)^n$ that is not realizable by $\HS$ but every neighbour of $s$ is realizable by $\hypothesisclass$. If no such largest integer $n$ exists, then $k_o=\infty$.

\citet[][\S2.1]{bousquet2020proper} estimate the hollow star numbers of several well-known hypothesis classes.
The following result provides a lower-bound on the sample complexity of proper learning:
\begin{theorem}[Thm.~11, \citealt{bousquet2020proper}] 
\label{thm:sample-complexity-lowerbound}
Let $\epsilon \in (0,1 / 8)$ and $\delta \in (0,1/100)$. There exists a hypothesis class with VC dimension $d$ and $k_o = \infty$ for which we have $\samplecomplex{\epsilon}{\delta} \geq \frac{\tilde{c}}{\epsilon}(d\, \mathrm{Log} \frac{1}{\epsilon} + \mathrm{Log}\frac{1}{\delta})$ for a fixed numerical constant $\tilde{c} > 0$ where $\mathrm{Log}(x) = \max \{1, \log(x) \}$ for $x\geq 0$.
\end{theorem}
We now present the main result.
\begin{theorem}\label{mainthm}
\cref{conjecture:realizable-generalf} is false.
\end{theorem}

\begin{proof}

We prove the claim by contradiction. Pick $f$ and $c \ge 0$.
Let $\HS$ be a hypothesis class with finite VC dimension $d$ but infinite hollow star number, as shown to exist by \cref{thm:sample-complexity-lowerbound}. 

Let $\TheAlg$ be a proper learning algorithm for $\HS$,
let $n \ge d$, and assume, for the eventual purpose of obtaining a contradiction, 
that
$\SZB \le f(d)$ for all $\Dist$ and,
for all $s \in \dataspace^{n}$, 
$\EE \EmpRisk{s}{\Alg(s)} \le c\, {d}/{n}$
if there exists $h \in \hypothesisclass$ such that $\EmpRisk{s}{h}=0$.

Pick a realizable distribution $\Dist$. 
It follows from the above assumption and \cref{eq:ege-interpolate} that
\begin{equation*}
\EE [\Risk{\Dist}{\Alg(\SS)}]\leq 2 c\, \frac{d}{n} + 3\frac{f(d)}{n } = \frac{1}{n}(3f(d)+2 c\,d)
\end{equation*}
By Markov's inequality,  
\begin{equation*}
\Pr (\Risk{\Dist}{\Alg(\SS)} \geq \epsilon)\leq \frac{1}{n\epsilon}(3f(d)+2c\,d).
\end{equation*}
It follows that the sample complexity of proper learning $\HS$ satisfies
\[
\label{eq:sample-complexity-upperbound}
\samplecomplex{\epsilon}{\delta}\leq  \frac{1}{\epsilon \delta}(3f(d)+2c\,d).
\]
Now, fix $\delta \in (0,1/100)$ and fix a convergent sequence of $\epsilon_i \downto 0$. 
There exists $J$ such that, for all $i \geq J$,
\[\label{eq:whatever}
\frac{1}{\tilde{c} \delta}(3f(d)+2c\,d)  < d\, \mathrm{Log} \frac{1}{\epsilon_i} +  \mathrm{Log}\frac{1}{\delta} ,
\] 
 for $\tilde{c}$ as in \cref{thm:sample-complexity-lowerbound}.
Combining \cref{eq:whatever} with \cref{eq:sample-complexity-upperbound},
\begin{equation*}
\samplecomplex{\epsilon_i}{\delta}
     < \frac{\tilde{c}}{\epsilon_i}(d \, \mathrm{Log} \frac{1}{\epsilon_i} + \mathrm{Log}\frac{1}{\delta})
\end{equation*}
for $i \geq J$.
Simultaneously, from \cref{thm:sample-complexity-lowerbound}, it follows that 
$\samplecomplex{\epsilon_i}{\delta}\geq  \frac{\tilde{c}}{\epsilon_i}(d\, \mathrm{Log} \frac{1}{\epsilon_i} + \mathrm{Log}\frac{1}{\delta})$, a contradiction.
\end{proof}

\begin{remark}[CMI bounds in the agnostic case]
Building on work by \citet{asadi2018chaining} combining chaining and mutual information, \citet{CCMI20} combine chaining with the CMI-based approach to generalization bounds for deterministic learning algorithms. 
As a corollary, \citeauthor{CCMI20} show that chaining CMI leads to a $O(\sqrt{d/n})$ bound 
for the expected generalization error of learning VC  classes in the \emph{agnostic} case. 
It is interesting to note that classical one-step discretization arguments also lead to a spurious $\log n$ factor when analyzing the expected generalization error in VC classes. As here, chaining methods were used to remove these log factors and obtain the tightest bounds (on uniform convergence and then excess risk) for VC classes \citep[Thm. 8.3.23 and \S8.8]{HDPbook}.
\end{remark}

\section{Discussion}

In this note, we refute \cref{conjecture:realizable} and \cref{conjecture:realizable-constant} by \citet{steinke2020openproblem}. 
In light of our observations, we can modify these conjectures to obtain new conjectures:
\setcounter{problem}{0}
\begin{problem}
There is a constant $c > 0$ such that,
for every VC class $\hypothesisclass$, with dimension $d$, 
if there exists a \emph{proper} learning algorithm with the expected risk no greater than $c\,d/n$ for every realizable distribution $\Dist$, 
then
there exists a \emph{proper} learning algorithm $\Alg$ with $\SZB \le c\, d$ and $\EE \EmpRisk{\SS}{\Alg(\SS)} \le c\, d/n$ for every realizable distribution $\Dist$.
\end{problem}
In the setting of improper learning, we know that every class with VC dimension $d$ is learnable 
with expected risk $O(d/n)$. This suggests the following conjecture:
\begin{problem}
There is a constant $c > 0$ such that,
for every VC class $\hypothesisclass$, with dimension $d$, 
there exists a (possibly improper) learning algorithm $\Alg$  such that
$\SZB \le c\,d$ and $\EE \EmpRisk{\SS}{\Alg(\SS)} \le c\, d/n $ for every realizable distribution $\Dist$.
\end{problem}
We leave the resolutions of these conjectures as open problems.
If either of these conjectures is false, it would demonstrate that we cannot completely characterize the expected generalization error of proper learning for VC classes in the realizable case.

\printbibliography
\end{document}